%% file: main.tex
\documentclass{article}

\usepackage{algorithm}
\usepackage{algorithmic}

\usepackage[letterpaper, left=1in, right=1in, 
top=1in, bottom=1in]{geometry}

\usepackage[parfill]{parskip}
 \usepackage{breakcites}
 \usepackage[dvipsnames]{xcolor}
 \usepackage[hidelinks]{hyperref}
\usepackage{hyperref}
 \hypersetup{
 	colorlinks=true,
 	linkcolor=blue!70!black,
 	citecolor=blue!70!black,
 	urlcolor=blue!70!black}

\usepackage{microtype,bm}
\usepackage{xspace}

\usepackage{verbatim}
\usepackage{mathtools}
\usepackage{amsmath}
\usepackage{bbm}
\usepackage{amsfonts}
\usepackage{amssymb}
\usepackage{xpatch}
\usepackage{enumitem}
\usepackage{amsthm}
\usepackage{microtype}
\usepackage{graphicx}
\usepackage{subfigure}
 \usepackage{hyperref}

\theoremstyle{definition}  

\newtheorem{lemma}{Lemma}

\newtheorem{fact}{Fact}

\theoremstyle{plain}

\xpatchcmd{\proof}{\itshape}{\normalfont\proofnameformat}{}{}
\newcommand{\proofnameformat}{\bfseries}

\usepackage{prettyref}
\newcommand{\pref}[1]{\prettyref{#1}}
\newcommand{\pfref}[1]{Proof of \prettyref{#1}}
\newcommand{\savehyperref}[2]{\texorpdfstring{\hyperref[#1]{#2}}{#2}}
\newrefformat{eq}{\savehyperref{#1}{\textup{(\ref*{#1})}}}
\newrefformat{eqn}{\savehyperref{#1}{Equation~\ref*{#1}}}
\newrefformat{lem}{\savehyperref{#1}{Lemma~\ref*{#1}}}
\newrefformat{obs}{\savehyperref{#1}{Observation~\ref*{#1}}}
\newrefformat{def}{\savehyperref{#1}{Definition~\ref*{#1}}}
\newrefformat{line}{\savehyperref{#1}{line~\ref*{#1}}}
\newrefformat{thm}{\savehyperref{#1}{Theorem~\ref*{#1}}}
\newrefformat{corr}{\savehyperref{#1}{Corollary~\ref*{#1}}}
\newrefformat{sec}{\savehyperref{#1}{Section~\ref*{#1}}}
\newrefformat{app}{\savehyperref{#1}{Appendix~\ref*{#1}}}
\newrefformat{ass}{\savehyperref{#1}{Assumption~\ref*{#1}}}
\newrefformat{ex}{\savehyperref{#1}{Example~\ref*{#1}}}
\newrefformat{fig}{\savehyperref{#1}{Figure~\ref*{#1}}}
\newrefformat{alg}{\savehyperref{#1}{Algorithm~\ref*{#1}}}
\newrefformat{rem}{\savehyperref{#1}{Remark~\ref*{#1}}}
\newrefformat{conj}{\savehyperref{#1}{Conjecture~\ref*{#1}}}
\newrefformat{prop}{\savehyperref{#1}{Proposition~\ref*{#1}}}
\newrefformat{proto}{\savehyperref{#1}{Protocol~\ref*{#1}}}
\newrefformat{prob}{\savehyperref{#1}{Problem~\ref*{#1}}}
\newrefformat{claim}{\savehyperref{#1}{Claim~\ref*{#1}}}

\input{fsm_stoch_macros}

\definecolor{darkred}{rgb}{0.6,0,0}

\newcommand{\msvrg}{T}

 \author{
 Yossi Arjevani \\
 NYU\\
 \texttt{yossi.arjevani@gmail.com} \\
 \and
 Amit Daniely\\
 The Hebrew University, Google Research\\
 \texttt{amit.daniely@mail.huji.ac.il} \\
 \and
 Stefanie Jegelka\\
 Massachusetts Institute of Technology\\
 \texttt{stefje@csail.mit.edu}
 \and
 Hongzhou Lin\\
 Massachusetts Institute of Technology\\
 \texttt{hongzhou@mit.edu} \\
 }	

\date{}
\begin{document}

\title{On the Complexity of Minimizing Convex Finite Sums Without \\
Using the Indices of the Individual Functions}

\maketitle



%
%





\begin{abstract}
	Recent advances in randomized incremental methods for minimizing $L$-smooth 
	$\mu$-strongly convex finite sums have culminated in {tight} complexity 
	bounds of $\tilde{O}((n+\sqrt{n L/\mu})\log(1/\epsilon))$ and 	
	$O(n+\sqrt{nL/\epsilon})$, where $\mu>0$ and $\mu=0$, respectively, and $n$ 
	denotes the number of individual functions. Unlike incremental methods, 
	stochastic methods for finite sums do not rely on an explicit knowledge of
	which individual function is being addressed at each iteration, and as 
	such, must perform at least $\Omega(n^2)$ iterations to obtain 
	$O(1/n^2)$-optimal solutions. In this work, we exploit the finite noise 
	structure of finite sums to derive a matching $O(n^2)$-upper bound under 
	the global oracle model, showing that this lower bound is indeed tight. 
	Following a similar approach, we propose a novel 
	adaptation of SVRG which 
	is both \emph{compatible with stochastic oracles}, and achieves complexity 
	bounds of $\tilde{O}((n^2+n\sqrt{L/\mu})\log(1/\epsilon))$ and  
	$O(n\sqrt{L/\epsilon})$, for $\mu>0$ and $\mu=0$, respectively. Our bounds 
	hold w.h.p. and match in part existing lower bounds of $\tilde{\Omega}(n^2+
	\sqrt{nL/\mu}\log(1/\epsilon))$ and $\tilde{\Omega}(n^2+\sqrt{nL/\epsilon})$, for
	$\mu>0$ and $\mu=0$, respectively.
	\end{abstract}

\input{sec1-intro.tex}

\input{sec2-settings.tex}

\input{sec3-existing-approach.tex}

\input{sec4-estimator.tex}

\input{sec5-sfo.tex}

\input{sec6-discuss.tex}
	
 \newpage

\bibliographystyle{alpha}
\bibliography{mybib}

 \newpage

\input{appendix}
\end{document}

%% file: fsm_stoch_macros.tex
\def\ddefloop#1{\ifx\ddefloop#1\else\ddef{#1}\expandafter\ddefloop\fi}
\def\ddef#1{\expandafter\def\csname 
	bb#1\endcsname{\ensuremath{\mathbb{#1}}}}
\ddefloop ABCDEFGHIJKLMNOPQRSTUVWXYZ\ddefloop
\def\ddef#1{\expandafter\def\csname 
	#1\endcsname{\ensuremath{{\bm{#1}}}}}
\ddefloop ABCDEFGHIJKLMNOPQRSTUVWXYZ\ddefloop
\def\ddef#1{\expandafter\def\csname 
	#1\endcsname{\ensuremath{{\bm{#1}}}}}
\ddefloop abcdefghijklmnopqrstuvwxyz\ddefloop
\def\ddefloop#1{\ifx\ddefloop#1\else\ddef{#1}\expandafter\ddefloop\fi}
\def\ddef#1{\expandafter\def\csname 
	b#1\endcsname{\ensuremath{\mathbf{#1}}}}
\ddefloop ABCDEFGHIJKLMNOPQRSTUVWXYZ\ddefloop
\def\ddef#1{\expandafter\def\csname 
	c#1\endcsname{\ensuremath{\mathcal{#1}}}}
\ddefloop ABCDEFGHIJKLMNOPQRSTUVWXYZ\ddefloop
\def\ddef#1{\expandafter\def\csname 
	f#1\endcsname{\ensuremath{\mathfrak{#1}}}}
\ddefloop ABCDEFGHIJKLMNOPQRSTUVWXYZ\ddefloop
\def\ddef#1{\expandafter\def\csname 
	h#1\endcsname{\ensuremath{\widehat{#1}}}}
\ddefloop ABCDEFGHIJKLMNOPQRSTUVWXYZ\ddefloop
\def\ddef#1{\expandafter\def\csname 
	hc#1\endcsname{\ensuremath{\widehat{\mathcal{#1}}}}}
\ddefloop ABCDEFGHIJKLMNOPQRSTUVWXYZ\ddefloop
\def\ddef#1{\expandafter\def\csname 
	t#1\endcsname{\ensuremath{\widetilde{#1}}}}
\ddefloop ABCDEFGHIJKLMNOPQRSTUVWXYZ\ddefloop
\def\ddef#1{\expandafter\def\csname 
	tc#1\endcsname{\ensuremath{\widetilde{\mathcal{#1}}}}}
\ddefloop ABCDEFGHIJKLMNOPQRSTUVWXYZ\ddefloop

\newcommand{\EE}{\mathbb{E}}
\newcommand{\RR}{\mathbb{R}}

\newcommand{\PP}{\mathbb{P}}

\newcommand{\bmu}{\boldsymbol{\mu}}

\renewenvironment{proof}{\par\noindent{\bf Proof\ }}{\hfill\BlackBox\\[2mm]}
\newcommand{\BlackBox}{\rule{1.5ex}{1.5ex}}

\newcommand{\todoc}[1]{\ignorespaces}

\newcommand{\bigtO}[1]{\tilde{\mathcal{O}}{\left(#1\right)}}

\newcommand{\defeq}{\coloneqq}

\DeclarePairedDelimiter{\prn}{(}{)}
\DeclarePairedDelimiter{\nrm}{\|}{\|}
\DeclarePairedDelimiter{\tri}{\langle}{\rangle}

\usepackage{nicefrac}

\newsavebox\CBox 

\usepackage{color}
\usepackage{soul}

\newcommand{\eps}{\mbox{$\varepsilon$}}

\newcommand{\defoo}{~\dot{=}~}

\newcommand{\rnd}[1]{\operatorname{rnd}\prn*{#1}}

\newcommand{\NN}{\mathbb{N}}

\newcommand{\strcvx}{\mu} 
\newcommand{\query}{B}

\newtheorem{thm}{Theorem}
\newtheorem{corr}[lemma]{Corollary}
\theoremstyle{definition}

\theoremstyle{remark}

\newcommand{\Cite}[2]{[\cite{#1},#2]}

\newcommand{\funclass}{\cF}
\newcommand{\algclassk}{\cA(\oracle,k)}
\newcommand{\alg}{\mathsf{A}}

%% file: sec1-intro.tex

\section{Introduction} \label{sec:intro}

	Many tasks in machine learning and statistics reduce to finite-sum 	
	minimization problems of the~form
	\begin{align} \label{opt:fsm}
		\min_{\w\in\RR^d} F(\w) \coloneqq \frac1n\sum_{i=1}^n f_i(\w),
	\end{align}
	where the individual functions $f_i$ are $L$-smooth and 
	$\strcvx$-strongly convex or convex. The large datasets often 
	encountered 
	in modern 
	applications have led to high interest in optimization methods which 
	can \emph{efficiently} cope with a large numbers of individual functions. 
	
	In this work, we measure efficiency of optimization algorithms through 
	the 
	framework of oracle complexity. Concretely, we assume the existence of 
	an 
	external procedure, typically referred to as an \emph{oracle}, which 
	upon 
	receiving a query, reveals some information about the function at hand 
	(e.g., function values, gradients or Hessians). The 
	\emph{oracle complexity} of a given optimization algorithm is then the number of oracle calls required to obtain an 
	$\epsilon$-optimal solution, i.e., a point $\w\in\RR^d$ that satisfies 
	$F(\w)-\min_{\w\in\RR^d} F(\w) <\epsilon$ w.h.p. or 
	$\EE[F(\w)-\min_{\w\in\RR^d} F(\w)]<\epsilon$, where the expectation is over randomness in the algorithm 
	and oracle.

	Throughout, we focus on two types of oracles for finite sums: 
	\emph{incremental} oracles, which allow one to control which individual 
	function is being referred to in each iteration, and \emph{stochastic} 
	oracles, which provide information on a randomly chosen 
	$f_i$---without revealing its index $i$ (see \pref{sec:settings} for a 
	formal exposition). Although the difference between these two types of 
	oracle may seem insignificant at first glance, the optimal attainable 
	performances of \emph{randomized incremental methods} (compatible with 
	incremental oracles), such as SAG \cite{schmidt2013minimizing}, SVRG 
	\cite{johnson2013accelerating} and  SDCA \cite{shalev2013stochastic}, are
	significantly better than those attainable by \emph{stochastic methods} 
	(compatible with stochastic oracles).\footnote{We follow here a 
	nomenclatural convention that distinguishes cases where deterministic 
	problems are addressed via random methods, as in  (\ref{opt:fsm}), 
	from cases that are inherently stochastic, as in  (\ref{opt:stoch})~below. 
	This distinction is made formal in \pref{sec:settings}.}

	\textbf{First-order oracles.} A notable example where incremental methods outperform stochastic 
	methods is the case of first-order oracles, which provide 
	function 
	values and gradients. When the $f_i$ are smooth and strongly convex, SAG, 
	SDCA 
	and SVRG enjoy exponential rates of $O(\log(1/\epsilon))$ 
	(disregarding 
	other problem parameters for now), while the stochastic vanilla SGD obtains 
	significantly 
	slower rates of	$O(\text{poly}(1/\epsilon))$, e.g., 
	\cite{shalev2011pegasos}. The fact that randomized incremental methods 
	can 
	converge exponentially fast is not altogether trivial. In particular, 
	it 
	implies that the variance of the iterates must also decay 
	exponentially 
	fast, which explains why such methods are often referred to as 
	\emph{variance-reduced methods}. 
	
	The main algorithmic idea behind variance reduction is to wisely 
	incorporate past information acquired from the oracle (see 
	\cite{johnson2013accelerating}). This naturally raises the following 
	question. \emph{Is it possible to use the same idea to design stochastic,
	rather than incremental, variance-reduced methods that achieve 
	exponential convergence rates for finite-sum problems? }
	
	In this work, we answer this question in the affirmative. 
	Specifically, we exploit the unique finite noise structure of finite sums 
	to form an estimator that recovers the full gradient of $F$ w.h.p.. 
	Based on this estimator, we propose a novel adaptation of SVRG which 
	does not rely on the indices of the individual functions. Combined with the 
	Catalyst acceleration scheme \cite{lin2015universal}, this yields upper 
	complexity bounds of 
	\begin{align} \label{upper_bounds}
	\tilde{O}\prn*{n^2+n\sqrt{{L}/{\epsilon}}} \quad
	\text{ and } \quad
	\tilde{O}\prn*{(n^2+n\sqrt{L/\strcvx})\log(1/\epsilon)},
	\end{align}
	for $\strcvx=0$ and $\strcvx>0$, respectively, which hold w.h.p..
	In particular, this shows that the additional finite-sum structure 
	can indeed be used to achieve exponential rates for stochastic 
	methods in 	the strongly convex case (and $O(\sqrt{1/\epsilon})$ rate in 
	the smooth convex case), which is impossible for general-purpose 
	stochastic methods 
	\cite{nemirovskyproblem,agarwal2009information,raginsky2011information}.
	Perhaps surprisingly, although this rate cannot be achieved through a 
	naive empirical average gradient estimator, as we prove in 
	Section~\ref{sec:related}, a simple rounding correction is all it needs to 
	obtain an exponential convergence rate w.h.p. (in the strongly convex 
	case).

	Although the quadratic dependence on $n$ may seem rather pessimistic, 
	any stochastic method that obtains an $O(1/n^2)$-optimal solution must issue
	at least $\Omega(n^2)$ oracle queries \cite{arjevani2017limitations}.  
	In fact, in the class of stochastic first-order	methods, one 
	\emph{cannot} hope to 
	perform better than
	\begin{align}
	\tilde{O}\prn*{n^2+\sqrt{{nL}/{\epsilon}}} \quad
	\text{ and } \quad
	\tilde{\Omega}(n^2 + \sqrt{nL/\strcvx}\log(1/\epsilon)),
	\end{align}
	for $\strcvx=0$ and $\strcvx>0$, respectively (see 
	(\ref{ineq:orcsfn_lb}) below for details). The complexity bounds 
	stated in  (\ref{upper_bounds})  provide, therefore, a fair idea of 
	the complexity of obtaining high-accuracy solutions in applications 
	where the indices of the individual functions are not known, or not 
	used. 
	
	\textbf{Global oracles.} The $\Omega(n^2)$-lower bound in \cite{arjevani2017limitations} 
	applies in fact to the much broader class of \emph{global oracles}, under 
	which 
	a complete specification of a randomly chosen 
	individual function is provided. Clearly, any local 
	information, such as gradients, Hessians, and other high-order derivatives, 
	as well as global information, such as the minimizers of $f_i$ along a 
	given direction (i.e., steepest descent steps), can be extracted through 
	global oracles. 
	An important 
	setting in which one is granted such `privileged access' is 
	empirical risk minimization (ERM),
	\begin{align}
	F(\w) = \sum_{i=1}^n \ell(\w; (\x_i,y_i)),
	\end{align}
	where $\ell$, the \emph{loss function}, is known a priori, and 
	$(\x_i,y_i)$ denote the training samples. Since each individual 
	function is fully parameterized by its associated sample, 
	this implies global access to the randomly chosen individual functions. 
	In this work, by using a similar idea to the one we use for first-order 
	oracles, 
	we show that the global $O(n^2)$-lower complexity bound is tight.

	Our contributions can be summarized as follows: 
	\begin{itemize}[leftmargin=1em]
		\item We show that (unlike existing variance-reduced methods 
		which directly rely on the indices of the individual functions) it 
		is 
		possible to minimize smooth strongly-convex finite-sum 
		problems---without using the indices---and still enjoy exponential 
		convergence rates. 
		
		\item To this end, we combine the SVRG method with a novel 
		biased quantized gradient estimator which recovers the average 
		function 
		w.h.p.. We further apply the Catalyst framework to improve the 
		dependence on 
		the condition number and to extend the proposed algorithm to 
		convex 
		functions (which are not necessarily strongly-convex).

		\item We prove that variance reduction cannot be obtained by 
		using 
		SGD or SVRG with a naive gradient estimator, as the rate at which 
		the 
		variance decays in this case is too slow, which in turn leads to  
		polynomial convergence rates. Thus, perhaps surprisingly, the 
		seemingly negligible quantization correction we propose is 
		essential 
		for obtaining exponential (rather than polynomial) rates.

		\item Using an estimator similar to the quantized gradient estimator, we show 
		that 
		the $\Omega(n^2)$-lower bound established in 		
		\cite{arjevani2017limitations} for stochastic global methods for 
		finite sums is essentially tight.
		
	\end{itemize}
	
	The paper is structured as follows. In \pref{sec:settings}, we 
	introduce the main ingredients of the oracle complexity framework 
	and 
	discuss relevant lower complexity bounds. \pref{sec:related} surveys 
	existing approaches for minimizing finite sums and their limitations in the 
	context of stochastic oracles. In \pref{sec:estimator}, we present our 
	quantized estimator for categorical random variables. With this, we 
	establish in  \pref{sec:app} a tight upper bound for 
	stochastic global oracles and present a variant of SVRG, which does not use 
	indices explicitly, and extend its applicability through the Catalyst 
	framework.

%% file: sec2-settings.tex
	\section{Setup} \label{sec:settings}
	\newcommand{\fsm}{\Sigma}
	\newcommand{\fsmlL}{\Sigma^L_\strcvx}
	\newcommand{\fsml}{\Sigma_\strcvx}
	\newcommand{\fsmL}{\Sigma^L}
	We study the problem of finding an $\epsilon$-optimal solution in the 
	framework of oracle complexity. First, we briefly 
	review its main components. For further details 
	on information-based complexity, see 
	\cite{traub1980general}.
	
	\textbf{Function classes\quad} We denote by $\fsm$ the class of 
	generic 
	finite-sum problems of the form (\ref{opt:fsm}). Similarly, let $\fsml$ denote the 
	class of finite sums with $\strcvx$-strongly convex individual functions, 
	where for any $i\in[n]$ and $\w,\u\in\RR^d$,
	\begin{align}
	f_i(\u) \ge f_i(\w) + \tri*{ \nabla f_i(\w), \u-\w} + 
	\frac{\strcvx}{2}\nrm{ \u-\w }^2.
	\end{align}
	If $\strcvx=0$, then $f_i$ is merely convex. 
	Likewise, we use $\fsmlL$ to indicate that the individual functions are 
	further assumed to 
	be $L$-smooth, i.e., for any $i\in[n]$ and $\w,\u\in\RR^d$,
	\begin{align}
		f_i(\u) \le f_i(\w) + \tri*{ \nabla f_i(\w), \u-\w} + 
		\frac{L}{2}\nrm{ \u-\w }^2.
	\end{align}
	Lastly,  we always assume that the initial suboptimality of 
	$F$ at the initialization point $\w_0$ is bounded from above by
	\begin{align}
	F(\w_0)-F^*\le \Delta,
	\end{align}
	where $\Delta$ is some positive real scalar assumed to be known. 
	
	\textbf{Oracle Classes\quad} Different ways of accessing a given 
	optimization problem are modeled through different oracles, 
	and these, in turn, accommodate 
	different classes of optimization methods. In this work, 
	we consider the following four distinct types of oracles:
	\newcommand{\orcifn}{\mathsf{I}_{\nabla}}
	\newcommand{\orcif}{\orcifn(\w,i)}
	\newcommand{\orcign}{\mathsf{I}_{f}}
	\newcommand{\orcig}{\orcign(i)}
	\newcommand{\orcsfn}{\mathsf{S}_{\nabla}}
	\newcommand{\orcsf}{\orcsfn(\w)}
	\newcommand{\orcsg}{\mathsf{S}_{f}}
	\newcommand{\oracle}{\mathsf{O}}
	\begin{enumerate}[leftmargin=0.3cm]
		\item \textbf{Incremental first-order oracle} defined with a parameter~$\query$ as		\label{item:inc_orc}
		\begin{align}
			&\orcifn:(\RR^d)^\query \times [n] \to(\RR\times 
			\RR^d)^\query:\\
			& (\w_1,..\w_\query,i)\mapsto \nonumber \\
		&\qquad\qquad\prn*{	f_i(\w_1),\nabla f_i(\w_1),..,f_i(\w_\query),\nabla 
			f_i(\w_\query)}.\nonumber 
		\end{align}
		The oracle $\orcifn$ allows the user to obtain the gradient of an 
		individual function they desire, at $\query$ different points (note that 
		we avoid explicitly stating $\query$ in $\orcifn$ to allow a cleaner 
		notation).  Having $\query>1$ is	necessary for implementing the SVRG method 
		in \pref{sec:app}. 
		
		\item \textbf{Incremental global oracle} defined by
		\begin{align}
			\orcign:[n] &\to (\RR^d \mapsto \RR): 
			~ i \,\,   \mapsto   \,\, f_i.
		\end{align}
		In this oracle model, the entire function $f_i$ is accessible to the 
		user. Hence one can
		simply issue $n$ queries to get $f_1,\dots,f_n$, and then return
		an exact minimizer of $F$. Under the $\orcign$-model, one completely 
		disregards the cost of computing local information, such as 
		gradients. 
		It can be shown that $\Omega(n)$ of queries are necessary to obtain 
		solutions of sufficiently high accuracy (e.g., Lemma~2 in 
		\cite{arjevani2016dimension}). 
	
		\item \textbf{Stochastic first-order oracle} defined with a parameter 
		$\query$ as 		
		\begin{align}		\label{eq:stoch_orc}
		&\orcsfn:(\RR^d)^{\query} \to(\RR\times \RR^d)^{\query}:\\
		&(\w_1,\dots,\w_\query)\mapsto \nonumber\\
		&\qquad \qquad \prn*{f_i(\w_1),\nabla f_i(\w_1),.., f_i(\w_\query),\nabla 
		f_i(\w_\query)}, 
		\nonumber\\
		 &\text{ where } i\sim\text{Unif}([n])\nonumber.
		\end{align}
		The class of stochastic oracles, the main focus of this work, is used 
		to model methods which do not or cannot explicitly rely on the individual 
		function index. This happens for example when `non-enumerable' data 
		augmentation is used (see, e.g., \cite{loosli2007training}). 
		Stochastic 
		methods are used more broadly to address stochastic optimization 
		problems of the form $\bar{F}=\EE_\xi[f(\w;\xi)]$, where one is 
		given 
		access to a first-order estimate of the gradient of $\bar{F}$ at 
		given 
		point $\w$ through a randomly drawn  $f'(\w,\xi)$. Note that 
		existing 
		variance 
		reduced methods are not directly implementable with  stochastic oracles, as we discuss in \pref{sec:related}.

		\item \textbf{Stochastic global oracle} defined by 
		\begin{align}
		\orcsg \text{ returns } f_i, \text{ where 
		} 	i\sim\text{Unif}([n]).
		\end{align}
		This oracle is used to study the fundamental statistical limitations 
		of minimizing finite-sum problems, as all other computational 
		aspects 
		are disregarded under this oracle model. As mentioned earlier, an 
		important instance of this setting is ERM. In this case, 
		the 
		global stochastic oracle complexity is typically referred to as 
		\emph{sample complexity}.
		
	\end{enumerate}
	
	\textbf{Minimax Oracle Complexity\quad} Next, we next  general oracle (minimax) complexity. Given a class $\funclass$
	of functions  and a suitable oracle~$\oracle$, we denote by 
	$\cA(\oracle,k)$ the class of all optimization algorithms that 
	access instances in $\funclass$ by issuing at most $k$ 
	$\oracle$-queries.\footnote{Strictly speaking, when studying 
	complexity 
	of optimization algorithms, one must carefully define what is meant by 
	`algorithm' and `oracle'. We shall not need this level of 
	formality in this work.}
	Let $\w_{A(f)}$ be the final iterate 
	returned by algorithm $\alg\in\cA(\oracle,k)$ when applied to an $f\in\cF$. With this, we define two related 
	notions of minimax complexity:
	\newcommand{\comp}{\mathfrak{M}}
	\begin{align}\label{eq:comp_exp}
		\comp_{\funclass,\oracle}(\epsilon) \defoo \inf\{  
		k\in\NN & \,\, |~ 
		\exists A\in\algclassk,\text{ s.t. } \nonumber\\
		& \sup_{f\in\funclass}\EE[ 
		f(\w_{A(f)}) - f^* ] \le \epsilon \},\\
	\label{eq:comp_whp}
		\comp_{\funclass,\oracle}(\epsilon,\delta) \defoo \inf \{ 
		k\in\NN& \,\,|~ \exists A\in\cA(\oracle,k),\text{ s.t.
		} \nonumber\\
		& \sup_{f\in\funclass} \PP\prn{f(\w_{A(f)}) - 
		f^*  > \epsilon} < \delta \}, 
	\end{align}
	where, here and throughout, $f^*$ denotes the infimum of 
	$f$ over its domain. 
	
	A straightforward application of Markov's inequality shows that the first 
	complexity notion \pref{eq:comp_exp}, which holds in expectation, implies 
	the second complexity notion \pref{eq:comp_whp}, which holds w.h.p.. We 
	note in passing that, in spite of being a well-accepted framework in 
	the field of continuous optimization, oracle complexity does not take into 
	account the computational resources required to implement and process 
	oracle calls, and should therefore be regarded as a lower 
	bound on the real computational complexity.

	\textbf{Lower Oracle Complexity Bounds\quad} Some oracles are 
	more 
	expressive 
	than others. For example, it is straightforward to implement the 
	stochastic first-order oracle $\orcsfn$ through the global oracle $\orcsg$. 
	Therefore, by definitions (\ref{eq:comp_exp}) and (\ref{eq:comp_whp}) above 
	and by	\Cite{arjevani2017limitations}{ Theorem 1}, we 	have 
	\begin{align}
	\comp_{\fsmlL,\orcsfn}(\epsilon)\ge 
	\comp_{\fsmlL,\orcsg}(\epsilon)\ge \tilde{\Omega}(n^2),
	\end{align}
	provided that $\epsilon\le O(1/n^2)$. Here and below, we omit the dependence 
	on the initial suboptimality $\Delta$ for simplicity. Likewise, $\orcifn$ 
	can be used to 
	implement 
	$\orcsfn$ be calling $\orcifn(\cdot,i)$ with $i\sim\text{Unif}([n])$, hence 
	\begin{align}
	\comp_{\fsmlL,\orcsfn}(\epsilon)\ge \comp_{\fsmlL,\orcifn}(\epsilon)\ge 
	\tilde{\Omega}(\sqrt{nL/\strcvx}\log(1/\epsilon)).
	\end{align}
	Combining both bounds gives
	\begin{align}\label{ineq:orcsfn_lb}
	\comp_{\fsmlL,\orcsfn}(\epsilon)\ge 	\tilde{\Omega}(
	n^2+\sqrt{nL/\strcvx}\log(1/\epsilon)).
	\end{align}
	As mentioned earlier, we also have 
	$\comp_{\fsmlL,\orcsfn}(\epsilon)\ge 
	\comp_{\fsmlL,\orcign}(\epsilon)\ge \tilde{\Omega}(n)$, but this additional 
	bound is absorbed in the $n^2$ term in (\ref{ineq:orcsfn_lb}). Using 
	similar arguments, one obtains a bound for the $L$-smooth case, 
	$\comp_{\fsmL,\orcsfn}(\epsilon)\ge 	
	\tilde{\Omega}(n^2+\sqrt{nL/\epsilon})$. Fact~\ref{fact} summarizes these  lower complexity bounds, for reference in later sections. 
	We leave a treatment of the h.p. counterparts to future work.
	\begin{fact}\label{fact}
	The following bounds hold for minimizing finite-sum functions.
	\begin{itemize}[leftmargin=1em]
	    \item For an incremental first order oracle 
	    \cite{woodworth2016tight,arjevani2016dimension}:  
	    \begin{align}
		    \comp_{\fsmlL,\orcifn}(\epsilon) &\ge 
		    \tilde{\Omega}(n+\sqrt{nL/\strcvx}\log(1/\epsilon)),\\
		    \comp_{\Sigma_{0}^{L},\orcifn}(\epsilon) &\ge 
		    \tilde{\Omega}(n+\sqrt{nL/\epsilon}).
	    \end{align}
	   \item For a stochastic finite sum oracle:
	  \begin{align}
	  	 \comp_{\fsmlL,\orcsfn}(\epsilon) &\ge 
	  	 \label{eq:lb_n2_fs_sc}
	  	 \tilde{\Omega}(n^2+\sqrt{nL/\strcvx}\log(1/\epsilon)),\\
	  	\comp_{\Sigma_{0}^{L},\orcsfn}(\epsilon) &\ge 
	  	\tilde{\Omega}(n^2+\sqrt{nL/\epsilon}).\label{eq:lb_n2_fs_nsc}
	  \end{align}
	\end{itemize}
	\end{fact}
	

%% file: sec3-existing-approach.tex
\section{Approaches for Minimizing Finite Sums} \label{sec:related}

	Next, we review existing approaches for minimizing finite-sum 
	problems via randomized incremental and stochastic methods. This 
	brief 
	survey also serves as a motivational exposition for the main question 
	we 
	ask in this work, namely, \emph{is it possible to apply variance-reduced 
	techniques using first-order stochastic oracles (defined in 
	\pref{eq:stoch_orc})}?
	
	For concreteness of the following discussion, let us consider the 
	class 
	$\fsmlL$ with $\strcvx>0$. A natural approach for addressing 
	finite-sum problems without relying on the indices is to re-express 
	(\ref{opt:fsm}) as a stochastic optimization problem,
	\begin{align} \label{opt:stoch}
		\min_{\w\in\RR^d} \EE_{i\sim\text{Unif}([n])}[f_i(\w)], 
	\end{align}
	and apply generic stochastic methods, such as SGD 
	\cite{robbins1951stochastic}. The oracle complexity of vanilla SGD is 
	$\text{poly}\prn{1/\epsilon}$ (e.g., \cite{nemirovski2009robust}), which, 
	although inferior to incremental methods in terms of $\epsilon$, does not 
	depend on $n$. This makes SGD particularly suited for settings where $n$ is 
	very large and one desires solutions of \emph{moderate} 
	accuracy.

	Other, more recent variants of SGD use, e.g., mini-batches, importance 
	sampling, or fixed step sizes, to achieve exponential convergence 	
	rates---but only up to a certain noise level (see
	\cite{moulines2011non,needell2014stochastic,needell2016batched,gower2019sgd}).
	It is possible to gradually reduce the convergence noise level, but 
	that would effectively imply polynomial complexity bounds (hence, 
	giving the same rates attainable by vanilla SGD). This should come 
	as no surprise---any general-purpose stochastic first-order methods 
	designed for \ref{opt:stoch}) is bound to polynomial rates	
	\cite{nemirovskyproblem,agarwal2009information,raginsky2011information}.
	
	In contrast to this, if an incremental first-order oracles is given, one 
	can compute the full gradient of $F$ at each iteration by simply 
	iterating over the $n$ individual functions, and then use vanilla 
	Gradient Descent (GD) and Accelerated Gradient Descent (AGD, 
	\cite{nesterov2004introductory}). This yields oracle complexity bounds 
	of
	\begin{align} \label{opt:gd_agd} 
		\tilde{O}\prn{nL/ \strcvx\log(1/\epsilon)} ~~\text{ and }~~
		\tilde{O}\prn{n\sqrt{L/\strcvx}\log(1/\epsilon)}, 
	\end{align} 
	for GD and AGD, respectively, where $\tilde{\cO}$ hides some logarithmic 
	factors in the problem parameters.

	Recently, the new class of variance-reduced methods, such as SAG, 
	SDCA, SVRG, SAGA \cite{defazio2014saga}, SDCA without duality 
	\cite{shalev2015sdca}, or MISO$\slash$Finito 	
	\cite{mairal2015incremental,defazio2014finito}, 
	was shown to enjoy significantly better rates of 
	$\bigtO{(n+L/\strcvx)\log(1/\epsilon)}$, or
	\begin{align} 
	\label{inc_rates} 
		\tilde{O}\prn{(n+\sqrt{nL/\strcvx})\log(1/\epsilon)} 
	\end{align} 
	when acceleration schemes are used 
	\cite{lin2015universal,shalev2016accelerated,allen2017katyusha}.
	This rate is tight and cannot be improved in the class of 
	incremental first-order methods 
	\cite{woodworth2016tight,arjevani2016dimension}.

	Despite their favorable rates, variance-reduce methods cannot be 
	directly implemented 
	through stochastic first-order oracles. Indeed, algorithms like 
	SAG/SAGA/MISO/SDCA keep a list of $n$ vectors in memory, one for each 
	individual function. The $i$th vector is then updated using 
	gradients of $f_i$, which are acquired throughout the optimization 
	process. 
	Clearly, one must know which individual function is being addressed at 
	each iteration in order to performs such updates. Another notable example 
	of a variance-reduced method is  SVRG, which generates new iterates 
	using 
	the full gradient at some reference point $\tilde{\w}$, i.e.,
	\begin{align*}
	\w_t &= \w_{t-1} - \eta \prn{\nabla f_{i_t}(\w_{t-1})  - \nabla 
		f_{i_t}(\tilde{\w}) + \nabla F(\tilde{\w}) }.
	\end{align*}
	When the stochastic oracle allows multiple simultaneous queries (i.e., $\query
	\ge 2$ in \pref{eq:stoch_orc}), it is possible to evaluate the term 
	$\nabla f_{i_t}(\w_{t-1})  - \nabla f_{i_t}(\tilde{\w})$. However, one 
	cannot evaluate the full gradient~$\nabla F(\tilde{\w})$ by simply
	iterating over the indices of the individual functions. 
	
	A natural idea to address this issue is to replace the full gradient by 
	an empirical estimator that simply averages the $m$ sampled
	gradients. Unfortunately, this naive estimator does not lead to the 
	desired exponential convergence.

	\begin{thm}\label{thm:naive estimator}
	Assume that in each iteration we make a fixed number of $m$
	stochastic oracle calls to form the empirical estimator
	\[ \widehat{\nabla F}(\w) = \frac{1}{m} \sum_{i=1}^{m} g_i(\w), \]
	where each $g_i(\w)$ is uniformly sampled among the individual gradients 
	$\nabla f_1(\w), \nabla f_2(\w), \dotsc, \nabla f_n(\w)$, and apply SGD and 
	SVRG with $\widehat{\nabla F}(\w)$ and constant stepsize $\eta$. Then we have the following lower bounds:
	\begin{itemize}[leftmargin=1em]
	    \item {\bf Lower bound in expectation:} There exists $F \in \Sigma_1^1$ 
	    such that for sufficiently small $\epsilon$, for any choice of $m$ and stepsize $\eta$, the number of 
	    stochastic oracle calls required to ensure $\EE[ F(\w) - F(\w^*)] \le 
	    \epsilon$ is at least $\Omega(1/\epsilon)$. 
	    \item {\bf Lower bound in high probability:} There exists $F \in 
	    \Sigma_1^1$ such that for sufficiently small $\epsilon$ and $\delta$, for any choice of $m \ge 2$ and stepsize $\eta$, the 
	    number of stochastic oracle calls required to ensure $\PP(F(\w) - 
	    F(\w^*) \le   \epsilon)$ with probability $1-\delta$ is at least 
	    $\Omega(1/\sqrt{\epsilon})$. 
	\end{itemize}
	\end{thm}
	We prove \pref{thm:naive estimator} by carefully tracing the random 
	iterates produced by SGD (with different step sizes) when applied on the 
	finite-sum function (for some even $n$)
	\begin{align*}
		F(w) \defeq   \frac{1}{n}((n/4)(w-1)^2+ 
		(n/4)(w+1)^2), ~w\in\RR.
	\end{align*}
	It is straightforward to show that one full iteration of SVRG is equivalent 
	to one SGD step. Hence both methods can be effectively addressed in the 
	same way.

	We note that the high probability bound is not directly implied by the 
	bound in expectation. The former requires one to guarantee that the 
	distribution of the iterates does not concentrate around the mean of the 
	iterates, as this sequence can converge exponentially fast. The main tool 
	we use to ensure \emph{anti-concentration} of the iterates is the Berry-Esseen 
	theorem. One may notice that the dependency on $\epsilon$ in the high probability lower bound is not as good as the one in the expectation lower bound, which may be improvable with a different proof technique.  Our main message here is 
	to show that the naive estimator does not provide exponential convergence rate.

%% file: sec4-estimator.tex
	\section{A Biased Estimator with 
	Quantization}\label{sec:estimator}
	
	It turns out that the high variance exhibited by the 
	naive gradient estimator that leads to the lower bounds above, can be addressed by a simple rounding procedure. We first describe 
	this procedure in a general setting:
	estimating the parameters of \emph{categorical} random 
	variables.

	We start by introducing relevant definitions. Given a real 
	number $a\in\RR$, we let $\rnd{a}$ denote the closest integer 
	to $a$, where by convention $\rnd{k+0.5} = k$ for $k \in \NN$.  
	A random variable $X$ is said to be $(q,n)$-{categorical} if 
	$X$ is discrete with finite support  
	$\{s_1,\dots,s_q\}\subseteq V$, where $V$ is some vector 
	space, and 
	\[ \PP(X=s_i)=\frac{n_i}{n},~i\in[q] ,\]
	for some $q$ nonnegative integers $n_1 ,\dots,n_q$ which sum up 
	to $n$, that is, $\sum_{i=1}^q n_i=n$. Clearly,	
	stochastic oracles over finite-sum functions induce a 
	categorical distribution on their potential set of 
	answers. 
	
	The next simple lemma is a key insight for our analysis.
	
		\begin{lemma}\label{lem:categorical_estimation}
		Let $X$ be a $(q,n)$-categorical distribution.
		Let $X_1,\ldots, X_m$ be i.i.d. samples of $X$ and let $Z_i$ be the empirical counter of category $i$ defined by
		\[ Z_i= \sum_{j=1}^{m} \mathbbm{1}_{X_j = s_i}.  \]
        If $m \ge 2n^2\log \left(\frac{2n}{\delta}\right)$, then with 
		probability at least $1 - \delta$ we have that for every $i 
		\in [q]$, 	
		\[ \rnd{\frac{n Z_i}{m}} = n_i. \]
	\end{lemma}
	\begin{proof}
	Define $p_i = \frac{k_i}{n}$. By Hoeffding's inequality, we have that for 
	every $i$,
	\[
	\PP\left( \left| \frac{Z_i}{m} - p_i\right| \ge \frac{1}{2n}  \right) \le 
	2\exp\left( - \frac{m}{2n^2}\right) \le \frac{\delta}{n},
	\]
	where the last inequality is due to the assumption that $m \ge 
	2n^2\log \left(\frac{2n}{\delta}\right)$. It follows now by the union bound 
	that
	\begin{align*}
	& \PP\left(\exists i\text{ such that } \left| \frac{Z_i}{m} - p_i\right| \ge 
		\frac{1}{2n}  \right) \\
	\le & \sum_{i=1}^{n} \PP\left(\left| \frac{Z_i}{m} 
		- 	p_i\right| \ge \frac{1}{2n}  \right) \le \delta.
	\end{align*}
	Thus, with probability at least $1 - \delta$, for all $i\in[q]$, we have $\left| 
	\frac{Z_i}{m} - p_i \right| < \frac{1}{2n}$, implying that
	\begin{align*}
    \left| \frac{n Z_i}{m} - n_i \right| < \frac{1}{2}  \implies 
		\rnd{\frac{n Z_i}{m}} = n_i. 
	\end{align*}
	\end{proof}
	Therefore, with appropriate quantization, we can recover 
	the distribution parameters with high probability. More 
	importantly, the estimator 
	\newcommand{\xquan}{\hat{X}^{\mathrm{qn}}}
	\begin{align}\label{eq:quane}
		\xquan = \frac{1}{n} \sum_{i=1}^q 
		 \rnd{\frac{n 	
				Z_i}{m}} s_i,	
	\end{align}
    satisfies $\xquan= \EE[X]$ with 
    probability at least $1-\delta$ as long as $m \ge 2n^2\log 
    \left(\frac{2n}{\delta}\right)$.  It is worth noting that the 
    quantized estimator is not unbiased. A simple example follows 
    by a 
    straightforward  computation for $n=q=3$ (see 
    \pref{sec:biased_est_proof} in the appendix for full details).
    

    We emphasize that one does not need to know the  support set 
    $\{s_1,\dots,s_q\}$  in advance to implement this 
    estimator; the counter $Z_i= \sum_{j=1}^{m} \mathbbm{1}_{X_j 
    = s_i}$ can be implemented on the fly. During the sampling 
    procedure, we keep in  memory the set of sample `types' seen up 
    to some point, and update it accordingly. 
    \pref{lem:categorical_estimation} implies that with probability 
    at least $1-\delta$ all different categories are seen 
    and the corresponding probabilities are recovered.

%% file: sec5-sfo.tex
	\section{Application to Stochastic Oracles}  \label{sec:app}
	Having presented our quantized estimator for categorical random 
	variables, we now use it to design optimization algorithms for 
	finite sums which do not rely on the indices of the 
	individual functions.	

	\textbf{Stochastic global oracle\quad }  We consider first the oracle 
	complexity of stochastic global oracle.
	In this case, a straightforward application of the quantized estimator 
	\pref{eq:quane} with support set $\{f_1,\dots,f_n\}$ yields the 
	following 
	upper complexity bound.
	\begin{thm} \label{thm:global_comp}
	The minimax complexity of a stochastic 
	global oracle for the finite-sum problem (\ref{opt:fsm}) is bounded
	by 
	\begin{align}
	\comp_{\fsm,\orcsg}(\epsilon,\delta) &\le
	{2n^2\log\left(\frac{2n}{\delta}\right)}.
	\end{align}
	\end{thm}
	By the lower complexity bound $\comp_{\fsm,\orcsg}(\eps) = 
	\Omega(n^2)$ established in \cite{arjevani2017limitations}, this bound stated in \pref{thm:global_comp} is tight 
	up to logarithmic factors. Also, note that the bound stated in 
	\pref{thm:global_comp} applies to \emph{any} 	type of 
	individual functions (including non-convex and non-smooth functions).


	\begin{algorithm}[tb]
	\caption{ Q-SVRG } 
	\label{algo}
	\begin{algorithmic}[1]
        \STATE {\bf Initialize} $\tilde{\w}_0$. 
        \FOR{$k =1, \cdots K$} 
            \STATE Set reference point $\tilde{\w} = \tilde{\w}_{k-1}$.
            \STATE Set $\tilde{\mu}$ to be the quantized estimator 
            $\widehat{\nabla F}^{\mathrm{qn}}(\tilde{\w})$ at $\tilde{\w}$
            \STATE Initialize inner iteration $\w_{0} = \tilde{\w}$.
            \FOR{$t =1, \cdots T$} 
            \STATE Call oracle  $\orcsfn$ at $(\w_{t-1},\tilde{\w})$ and update
	\begin{align*}
	\w_t &= \w_{t-1} - \eta \prn{\nabla f_{i_t}(\w_{t-1})  - \nabla 
		f_{i_t}(\tilde{\w}) + \tilde{\bmu} }.
	\end{align*}
            \ENDFOR
        \STATE Set $\tilde{\w}_k = \frac{1}{T} \sum_{t=1}^T \w_t$
        \ENDFOR
    \end{algorithmic}
	\end{algorithm}
	
	\newcommand{\QSVRG}{Q-SVRG\xspace}
   \textbf{Stochastic First-order Oracle\quad }  Our quantized estimator can 
   also be used to recover full gradients of $F$. This enables us to implement a 
   `quantized' variant of SVRG (\QSVRG), over $L$-smooth and 
   $\strcvx$-strongly convex individual 
   functions---which is compatible with the stochastic oracle $\orcsg$. A 
   better  dependence on the condition number is then 
   achieved by applying the Catalyst acceleration framework. Moreover, the 
   Catalyst framework also allows us to extend the scope of  \QSVRG to cases 
   where the individual function are only assumed to be convex, rather than 
   strongly convex.

	Each 
	iteration of SVRG starts by obtaining a full gradient $\tilde{\bmu} \defoo\nabla 
	F(\tilde{\w})$ of $F$ at some 
	reference point $\tilde{\w}$.
	Next, SVRG generates $m\in\NN$ iterates by setting 
	\begin{align}
	\w_t &= \w_{t-1} - \eta \prn{\nabla f_{i_t}(\w_{t-1})  - \nabla 
		f_{i_t}(\tilde{\w}) + \tilde{\bmu} },
	\end{align}
	where $\w_0\defoo\tilde{\w}$, $i_t\sim 
	\text{Unif}{([n])},~t=1,\dots,m-1$; and $\eta$ and $m$ are 
	assumed to be fixed throughout the optimization process. Lastly, one uses the average of the points
	$\{\w_0,\w_1,\dots,\w_{m-1}\}$ as the reference point for 
	the next iteration. 
	
	The difficulty in implementing SVRG via stochastic 
	first-order oracles is that one cannot simply form the full gradient at the 
	reference point by sequentially iterating over the individual functions.  
	To remedy this, we use the quantized gradient estimator which allows us to 
	recover the full gradient w.h.p.. This is formally stated  as follows (see 
	full proof in \ref{lem:full_gradient_proof}).
	\begin{lemma} \label{lem:full_gradient} 
		One can compute the full gradient of $F$ at a given point 
		with success probability  $1-\delta$ via 
		$2n^2\log\left(\frac{2n}{\delta}\right)$ $\orcsfn$-calls. 
	\end{lemma}
	With a union bound, \pref{lem:full_gradient} can be further used to 
	obtain $k$ exact gradients (see \ref{corr:k_full_gradients_proof} for 
	details). 	
	\begin{corr}\label{corr:k_full_gradients}
		One can compute the full gradients of $F$ at $k$ different points in 
		$\RR^d$ with success probability of $1-\delta$ via 	
		$2n^2k\log\left(\frac{2nk}{\delta}\right)$ $\orcsfn$-calls.
	\end{corr}
	\pref{corr:k_full_gradients} implies that w.p. at least $1-\delta$ 
	over the oracle randomness, we can implement $k$ full iterations of 
	SVRG using overall number of 	
	$k(2n^2\log\left(\frac{2nk}{\delta}\right)+m)$ 
	$\orcsfn$-oracle calls, from which we conclude the following result.
	\begin{lemma}\label{lem:svrg_rate}
		With notations as above,
		\begin{align*}
			\comp_{\fsmlL,\orcsfn}(\epsilon,\delta) = 		
			\tilde{O}\prn*{\prn*{n^2+L/\strcvx}\log\prn*{{1}/{\delta\epsilon}}}.
		\end{align*}
	\end{lemma}
		
	\begin{proof}
	The proof of the theorem follows by a direct application of	
	Theorem~1 in \cite{johnson2013accelerating}. In 
	particular, the same analysis 
	holds conditioned on the event that the full gradients are exactly 
	recovered at each reference point. Formally, given $0<\epsilon<\Delta/2$, and 
	$\delta\in(0,1)$, 
	we denote by $E_k$ the event that 
	all the $k$ first full gradients are exactly estimated. We have
	\begin{align*}
	\EE[F(\hat{\w}_k)-F^* | E_k] \le \alpha^k \Delta,
	\end{align*}
	where the convergence factor $\alpha$ is given by,
	\begin{align*}
	\alpha & = \frac{1}{\strcvx \eta (1-2L\eta )m}+ \frac{2L\eta}{1- 2L\eta}.
	\end{align*}
    Setting $\eta= 1/8L$ and $m=32L/\strcvx$ yields $\alpha = 2/3$. Therefore, by Markov's inequality, 
    \[ \PP(F(\hat{\w})-F^*> \epsilon | E_K) \le \frac{\alpha^k \Delta}{\epsilon} \]
    Then, setting $k =\log(2\Delta/(\delta\epsilon))/\log (1/\alpha)$, we have 
    \begin{equation}\label{eq:svrg-proof}
        \PP(F(\hat{\w})-F^*> \epsilon | E_K) \le \frac{\delta}{2}.
    \end{equation} 
    On the other hand, based on Lemma~{\ref{lem:full_gradient}}, by using $2n^2K\log\left(\frac{4nk}{\delta}\right)$ $\orcsfn$-oracle calls, we recover all the $k$ full gradients with probability at least $ \PP(E_K ) \ge 1- \frac{\delta}{2}$.
    Together with (\ref{eq:svrg-proof}), this yields
    \begin{align*}
         \PP(F(\hat{\w})&-F^*> \epsilon) \\
     \le & \PP(E_K^c) + \PP(F(\hat{\w})-F^*> \epsilon | E_K)  \PP(E_K) \le 
     \delta. 
    \end{align*} 
    Therefore the overall number of $\orcsfn$-oracle calls required to obtain an
    $\epsilon$-optimal solution is bounded from above by 
    \begin{align*}
        & 2n^2K\log\left(\frac{4nK}{\delta}\right)+km 
     =   \tilde{O} \prn*{(n^2+\frac{L}{\mu}) \log \prn*{\frac{\Delta}{\delta \epsilon}} }. 
    \end{align*}
\end{proof}        
    In other words, replacing the naive estimator by the quantized estimator is all we need to obtain exponential convergence result. We emphasize that this result does not contradict the lower bound in Theorem~\ref{thm:naive estimator} because Theorem~\ref{thm:naive estimator} relies on the explicit form of the empirical estimator. Before moving on, it is worth noting that our algorithm requires the stochastic oracle to 
    allow multiple simultaneous queries, i.e., $\query \ge 2$ in 
    (\ref{eq:stoch_orc}). 

	Next, to improve the dependence on condition number, we apply the 
	Catalyst acceleration framework 
	\cite{lin2015universal} to \QSVRG. Catalyst acts 
	as a wrapper that takes as input an algorithm that converges exponentially 
	fast and outputs an accelerated algorithm.
	
	In detail, at 
	iteration $k$, we replace the original objective function $F$ by an 
	auxiliary objective $G_k$ defined by
	\[ G_k(\w) = F(\w) + \frac{\beta}{2} \|\w - \u_k \|^2 ,\]
	where $\beta$ is a well-chosen regularization parameter and $\u_k$ is 
	obtained by extrapolating solutions of previous subproblems. We optimize 
	$G_k$ up to accuracy $\epsilon_k$ and use the solution to 
	warm-start the next subproblem.

	For the 
	complexity analysis of the Catalyst acceleration 
	framework, assume that a given optimization algorithm $\alg$ 
	converges exponentially fast for the smooth and strongly convex problems 
	$G_k$, i.e.,
	\begin{equation}\label{eq:linear}
	    G_k(x_t)-G_k^* \le (1-\tau_{})^t (G_k(x_0)-G_k^*).
	\end{equation} 
	Then, applying Catalyst on $\alg$ yields \cite{lin2015universal} a global 
	complexity bound of
	\[ \tilde{O} \prn*{ \frac{1}{\tau_{\alg}} \sqrt{\frac{\mu+\beta}{\mu}} \log 
	\prn*{\frac{1}{\epsilon}}}, \]
	for finding an $\epsilon$-optimal solution. We remark that $\beta$ is a 
	free parameter and hence can be chosen to minimize the overall complexity 
	bound.

	In our case, the convergence rate parameter is 
	\[ \frac{1}{\tau_{\alg}} =  \prn*{ 
		n^2+\frac{L+\beta}{\strcvx+\beta}}\log (1/\delta). \]
	
	Therefore, the total number of $\orcsfn$-calls is given by 
	\begin{align}
		\tilde{O}\prn*{ \prn*{ 
				n^2+\frac{L+\beta}{\strcvx+\beta}} 
			\sqrt{\frac{\strcvx+\beta}{\strcvx}} \log (1/\delta)\log 
			(1/\epsilon)  }.	
	\end{align}
	Minimizing the total number of $\orcsfn$-calls with respect to $\beta$, 
	yields $\beta= \max 
	\left \{0, \frac{L-(n^2+1)\strcvx}{n^2} \right \}$. Plugging in the value 
	of $\beta$ gives the desired accelerated complexity bound
	\[	\tilde{O}\prn*{\prn*{n^2+n\sqrt{\frac{L}{\mu}}}\log (1/\delta) \log (1/ \epsilon) }. \]
	For randomized incremental methods, the acceleration occurs in 
	the ill-conditioned regime where $L/\mu \ge n$. Here, due to the augmented 
	cost of evaluating the full gradient,  acceleration only occurs in the 
	extremely ill-conditioned regime in which $L/\mu \ge n^2$. 
	
	The Catalyst framework is also useful as a means of extending the 
	applicability of \QSVRG to smooth convex finite sums. This follows by the 
	fact that subproblems $G_k$ are always strongly convex when $\beta>0$.

	In this case, by \cite{lin2015universal}, if $\alg$ is an algorithm which 
	satisfies (\ref{eq:linear}), then by applying the Catalyst yields a global 
	complexity bound of
	\[ \tilde{O} \prn*{ \frac{1}{\tau_{\alg}} \sqrt{\frac{\beta}{\epsilon}} 
	\log \prn*{\frac{1}{\epsilon}}} \]
	for finding an $\epsilon$-solution. Again, minimizing the global complexity 
	with respect to the parameter $\beta$ yields 
	$\beta = \frac{L}{n^2 }$, by 
	which we obtain the following complexity bound
	\[ \tilde{O}\prn*{ \sqrt{{ L n^2 \log(1/\delta)}/{\epsilon}} 
	\log(1/\delta\epsilon) }.\] 
	Both complexity bounds are summarized as follows.
	\begin{thm} \label{thm:acc_rates}
	With notations as above, 
		\begin{align*}
		\comp_{\fsmlL,\orcsfn}(\epsilon,\delta) &= 		
		\tilde{O}\prn*{\prn*{n^2+n\sqrt{L/\strcvx}}\log\prn*{{1}/{\delta\epsilon}}},\\
		\comp_{\fsmL,\orcsfn}(\epsilon,\delta) &= 		
		\tilde{O}\prn*{\prn*{n^2+ n\sqrt{{ L }/{\epsilon}}} 
		\log(1/\delta\epsilon) 	}.
		\end{align*}
	\end{thm}
	The bounds stated in \pref{thm:acc_rates} are partly complemented by the 
	lower bounds given in \pref{eq:lb_n2_fs_sc} and \pref{eq:lb_n2_fs_nsc}, 
	namely,	$\comp_{\fsmlL,\orcsfn}(\epsilon)\ge 	
	\Omega(n^2+\sqrt{nL/\mu}\log(1/\epsilon))$ and 	
	$\comp_{\fsmL,\orcsfn}(\epsilon)\ge \Omega(n^2+\sqrt{nL/\epsilon})$. 
	Specifically, in both cases, the proposed SVRG variant is tight w.r.t. the 
	global term $O(n^2)$. However, the first-order term of \QSVRG
	corresponds to that of deterministic methods, namely, $n\sqrt{L/\strcvx}$ (in the 
	strongly convex case), whereas the first-order term of randomized 
	incremental methods is $\sqrt{nL/\strcvx}$, and thus 	misses a factor of 
	$\sqrt{n}$.

%% file: sec6-discuss.tex
\section{Discussion and Future Work}

In this paper, we showed that although current variance-reduced finite-sum 
methods directly rely on the indices of the individual functions, it is 
possible to achieve variance reduction for obtaining 
exponential convergence rates even without this knowledge. 
Although this variance reduction cannot be achieved 
by simply averaging over the gradients, a simple rounding procedure 
suffices to obtain an exponential convergence rate w.h.p.. 

The cost of not having access to or disregarding the indices of the 
individual functions (as is often done in practice) is an expensive 
$O(n^2)$-term in the upper complexity bound---which is inevitable for 
stochastic methods compatible with~$\orcsfn$. This leads to a factor of 
$n\sqrt{L/\strcvx}$ 
in the first-order term (rather than the $\sqrt{nL/\strcvx}$-factor 
exhibited 
by incremental methods) which we suspect is tight for stochastic methods. 
We 
leave addressing this gap to future work.

One limitation of our approach is the requirement of issuing two or more 
queries simultaneously (i.e., $K\ge2$ in \pref{eq:stoch_orc}). This 
assumption 
is necessary to compute the expression 
\[ \nabla f_{i_t}(\w) - \nabla f_{i_t}(\tilde{\w})\]
for the SVRG update rule. Replacing it by 
\[ \nabla f_{i_t}(\w) - \nabla f_{j_t}(\tilde{\w})\]
introduces additional variance and breaks the current analysis.
Since our quantized estimator is still applicable when $\query=1$, one 
can still implement GD or AGD and obtain an exponential convergence rate 
of $O(n^2 \sqrt{L/\mu} \log (1/\epsilon))$, which is nontrivial 
to achieve in the stochastic setting. 

That said, to the best of our knowledge, no variance reduction 
technique is applicable with $\query=1$. This leads to an interesting 
open question: is it possible to exploit the finite sum structure and achieve 
exponential convergence rate better than $O(n^2 \sqrt{L/\mu} \log 
(1/\epsilon))$ in the stochastic first-order setting with $\query=1$? 
Addressing 
this question will provide further understanding of the variance 
reduction 
technique.

\subsection*{Acknowledgments}
This research was supported by The Defense Advanced Research Projects Agency 
(grant number YFA17N66001-17-1-4039). The views, opinions, and/or findings 
contained in this article are those of the authors and should not be 
interpreted as representing the official views or policies, either expressed or 
implied, of the Defense Advanced Research Projects Agency or the Department of 
Defense.


%% file: appendix.tex
\section{Supplementary Material}

\subsection{Quantized Estimator is 
Biased}\label{sec:biased_est_proof}
Consider $n=q=3$, i.e. there are $3$ category, each having probability $1/3$. Now assume $m=5$, then all the possible couples (up to permutation) of $(Z_0,Z_1,Z_2)$ are $(5,0,0)$, $(4,1,0)$, $(3,2,0)$, $(3,1,1)$, $(2,2,1)$. The corresponding $(\rnd{\frac{nZ_0}{m}}, \rnd{\frac{nZ_1}{m}}, \rnd{\frac{nZ_2}{m})}$ are $(3,0,0)$, $(2,1,0)$, $(2,1,0)$, $(2,1,1)$, $(1,1,1)$. Note that all the couples sum up to $3$ expect, $(2,1,1)$. Thus the estimator is biased. (If it is unbiased, the sum of the expectation should be $3$, but here it is $>3$.)

\subsection{Proof of \pref{lem:full_gradient}}
\label{lem:full_gradient_proof}
	\begin{proof}
	Let $\g_1,\ldots, \g_n$ be the gradients of the $n$ individual 
	functions corresponding to some point in $\RR^d$, and let 
	$\{\g'_1,\ldots,\g'_q\}$ denote the set of distinct 
	gradients (note that $q\le n$, with strict 
	inequality if two functions share the same gradient). 
	Denote  $n_i= |\{j : \g_j = \g'_i\}|$. 	Note that the full gradient can 
	be equivalently expressed as:
	\[
	\g = \frac{1}{n}\sum_{i=1}^n\g_i = \frac{1}{n}\sum_{i=1}^q n_i 
	\g'_i.
	\]
	Let $\hat \g_1,\ldots \hat \g_m$ be the answers of the first-order 
	oracle, and let $\{\hat \g'_1,\ldots, \hat \g'_{\hat{q}}\}$ denote the 
	corresponding set of distinct gradients (here 	$\hat{q}\le q$ with 
	strict inequality if one of the gradients was not 
	sampled). We let 
	$Z_ i= |\{j : \hat \g_j = \hat \g'_i\}|$,
	and estimate the gradient through	
	\[
	\hat \g = \frac{1}{n}\sum_{i=1}^{\hat q} \rnd{\frac{n 
			Z_i}{m}} 	\hat \g'_i.
	\]
	By \pref{lem:categorical_estimation} we have that with probability at 
	least $1-\delta$ and up to permutation of the indices, for every $i\in 
	[q]$, $n_i = \rnd{\frac{n Z_i}{m}}$, in which case 
	\begin{align*}
		\hat \g =  \frac{1}{n}\sum_{i=1}^q n_i \g'_i = \g.
	\end{align*}
	\end{proof}

	\subsection{Proof of \pref{corr:k_full_gradients}}
	\label{corr:k_full_gradients_proof}
	\begin{proof}
	Bt \pref{lem:full_gradient}, $2n^2\log\left(\frac{2nk}{\delta}\right)$  
	$\orcsfn$-calls suffice to compute the full gradient at a given point 
	with failure probability of at most $\frac{\delta}{k}$. Hence, by the 
	union bound, $k$ full gradients can be obtained with failure 
	probability of at most $\delta$ by using 
	$2n^2k\log\left(\frac{2nk}{\delta}\right)$ 
	$\orcsfn$-calls.
	\end{proof}

\subsection{\pfref{thm:naive estimator}}
\begin{proof}
Assume $n$ is even, and we define $F(x) = \frac{1}{n}((n/4)(x-1)^2+ (n/4)(x+1)^2) = 
\frac{1}{2}(x^2+1)$. In this case, $L=\mu =1$ and the minimum $x^* =0$. We now consider applying gradient descent (GD) with the naive unbiased gradient estimator on~$F$.   

At iteration $k$, we sample $m$ stochastic oracles, which is equivalent to pick $m$ points $z_{k,1},\dots,z_{k,m} \in 
\{-1,1\}$ at random independently, then the update is given by
\begin{align*}
x_{k+1} &= x_k - \alpha g(x_k)\\
&= x_k -  \frac{\alpha}{m}\sum_{j=1}^m (x_k+z_{k,j})  \\
&= x_k - \alpha (x_k +\underbrace{\frac{1}{m}\sum_{j=1}^m z_{k,j}}_{z_k} )= 
(1 -\alpha )x_k - \alpha z_k  ,
\end{align*}
Note that $ m z_k + m  \sim 2B(m, 1/2)$, where $B(m,p)$ is the binomial distribution. Therefore,
\begin{align*}
     \EE[ \| x_{k+1} - x^*\|^2 ] & = (1-\alpha )^2x_k^2 + \alpha^2  \EE[z_k^2]  \\
     & = (1-\alpha  )^2x_k^2 + \frac{4\alpha^2 }{m^2} Var(B(m, 1/2))\\
     & = (1-\alpha )^2\| x_k -x^* \|^2 + \frac{\alpha^2 }{m}
\end{align*}
A simple telescopic summing yields, 
\begin{align*}
    \EE[\| x_k -x^*\|^2] = (1-\alpha)^{2k} \EE[\| x_0 -x^*\|^2] + \frac{\alpha^2}{m} \sum_{i=0}^{k-1} (1-\alpha )^{2i} =   (1-\alpha )^{2k} \EE[\| x_0 -x^*\|^2] + \frac{\alpha }{m}\frac{(1- (1-\alpha )^{2k})}{2-\alpha} 
\end{align*}
Note that $F(x_k) - F^* = \frac{1}{2} x_k^2 =\frac{1}{2} \| x_k -x^* \|^2$. Therefore, recall that $\Delta = F(x_0)-F^* = \frac{1}{2} x_0^2$, we have
\[ \EE[F(x_k)-F^*] = (1-\alpha)^{2k} \Delta +  \frac{\alpha }{2m}\frac{(1- (1-\alpha)^{2k})}{2-\alpha}  \]
In order to guarantee $\EE[F(x_k) - F^*] \le \epsilon$, it is necessary to have both 
\[ (1-\alpha)^{2k} \Delta \le \epsilon \quad \text{ and } \quad \frac{\alpha}{2m}\frac{(1- (1-\alpha)^{2k})}{2-\alpha} \le \epsilon\]
This implies 
\[ k \ge  \frac{\log \left( \frac{\Delta}{\epsilon} \right )}{-2\log (1-\alpha)} \quad \text{ and } \quad m \ge \frac{\alpha (1- (1-\alpha)^{2k})}{2(2-\alpha)\epsilon} \ge \frac{\alpha (1- \frac{\epsilon}{\Delta})}{2\epsilon}\]
Let $T$ denotes the total number of oracles. If $\alpha \ge 1/2$, then 
\[ T \ge m \ge \frac{(1- \frac{\epsilon}{\Delta})}{4\epsilon} \]
If $\alpha < 1/2$, then $-\log(1-\alpha) \le 2\alpha$
\[ T \ge km  \ge \frac{ (1- \frac{\epsilon}{\Delta}) \log \left( 
\frac{\Delta}{\epsilon} \right ) }{8 \epsilon}\]
Therefore in both cases the complexity of obtaining an $\epsilon$ solution of $F$ is lower bounded by $\Omega(1/\epsilon)$. 


{\bf High probability result:}\\
We consider the same function $F$. With out loss of generality, let's assume $x_0 > 0$. Note that $F(x_k)-F^* = \frac{1}{2}x_k^2$, bounding $\PP(F(x_k)-F^* > \sqrt{2\epsilon})$ is equivalent to bound $\PP(x_k > \sqrt{2\epsilon})$. On one hand, $x_k - (1-\alpha^k)x_0$ is a symmetric random variable, thus 
\[ \PP(x_k \ge (1-\alpha)^k x_0) \ge \frac{1}{2}. \]
Therefore, in order to guarantee a high probability result for $\delta < 1/2$, it is necessary to have 
\[ (1-\alpha)^k \le \sqrt{\frac{\epsilon}{\Delta}} \implies k \ge  \frac{\log \left( \frac{\Delta}{\epsilon} \right )}{-2\log (1-\alpha)}  \]
On the other hand, let $Z_{i,j}$ be i.i.d Bernoulli random variable (i.e probability $1/2$ take value $1$ or $-1$.) then
\begin{align}
x_k -(1-\alpha)^{k}x_0 = -\frac{\alpha }{m} \sum_{i=0}^{k-1} (1-\alpha)^i \sum_{j=1}^m Z_{i,j}
\end{align}
Thus,
\begin{align*}
    \PP \left (x_k -(1-\alpha)^{k}x_0 \ge \sqrt{2\epsilon} \right ) =  \PP \left ( \sum_{i=0}^{k-1} (1-\alpha)^i \sum_{j=1}^m Z_{i,j} \le \frac{ -m \sqrt{ 2\epsilon}}{\alpha} \right ) 
\end{align*}
From Berry-Esseen theorem \cite{berry1941accuracy,esseen1942,korolev2010upper}, for any independent variables $Y_i$ with $\EE[Y_i] =0$, $Var(Y_i) = \sigma_i^2$, we have for any $u$
\begin{equation}\label{eq:berry}
    \left | \PP \left ( \frac{1}{\sqrt{ \sum_{i=1}^I \sigma_i^2 }} \sum_{i=1}^I Y_{i} \le u \right )  - \phi \left( u \right ) \right | 
\le C \frac{ \sum_{i=1}^I \EE[|Y_i|^3]}{(\sum_{i=1}^I \sigma_i^2 )^{3/2}} 
\end{equation} 
where $\phi$ is the normal Gaussian cumulative distribution function, i.e. $\phi(u) = \PP_{X \sim \mathcal{N}(0,1)}(X \le u)$; $C$ is an absolute constant not larger than ${0.5129}$.
Note that $Var(Z_{i,j}) = 1$ and $\EE[|Z_{i,j}|^3] = 1$, the Berry-Esseen theorem in (\ref{eq:berry}) yields
\begin{equation}
    \left | \PP \left ( \frac{1}{\sqrt{ m \sum_{i=0}^{k-1} (1-\alpha)^{2i} }} \sum_{i=0}^{k-1} (1-\alpha)^i \sum_{j=1}^m Z_{i,j}  \le u \right )  - \phi \left( u \right ) \right | 
\le C \frac{ m \sum_{i=0}^{k-1} (1-\alpha)^{3i}}{( m \sum_{i=0}^{k-1} (1-\alpha)^{2i})^{3/2}} 
\end{equation} 
Hence,
\[ \PP \left (x_k -(1-\alpha)^{k}x_0 \ge \sqrt{\epsilon} \right ) \ge \phi \left ( \frac{-\sqrt{2m \epsilon}}{\alpha \sqrt{  \sum_{i=0}^{k-1} (1-\alpha)^{2i} }}  \right ) - C \frac{ m \sum_{i=0}^{k-1} (1-\alpha)^{3i}}{( m \sum_{i=0}^{k-1} (1-\alpha)^{2i})^{3/2}}  \]

Now we analyze the term on the r.h.s. one by one. 
Indeed,
\begin{align*}
    \frac{\sqrt{2m \epsilon}}{\alpha \sqrt{  \sum_{i=0}^{k-1} (1-\alpha)^{2i} }} & = \frac{\sqrt{2m \epsilon}}{\alpha \sqrt{\frac{1-(1-\alpha)^{2k}}{1-(1-\alpha)^2}}} = \frac{\sqrt{(2-\alpha) 2m \epsilon }}{ \sqrt{\alpha(1-(1-\alpha)^{2k}})} 
     \le 2\sqrt{\frac{2m\epsilon}{\alpha}}
\end{align*}
where the last inequality, we use the fact $(1-\alpha)^{2k} \le \frac{\epsilon}{\Delta} < \frac{1}{2}$ and assume that $\epsilon < \Delta/2$. Then
\[  \phi \left ( \frac{-\sqrt{2m \epsilon}}{\alpha \sqrt{  \sum_{i=0}^{k-1} (1-\alpha)^{2i} }}  \right ) \ge   \phi \left ( - 2\sqrt{\frac{2m\epsilon}{\alpha}} \right ).\]
The high level idea is that $m$ must be of order $\alpha/\epsilon$, otherwise this quantity is approximately 0.5 as $\epsilon \rightarrow 0$. More precisely, if 
\[ m \le \frac{\alpha}{2\sqrt{2\epsilon}} \implies \phi \left ( - 2\sqrt{\frac{2m\epsilon}{\alpha}} \right ) \ge \phi \left ( - \epsilon^{1/4} \right ) \rightarrow_{\epsilon \rightarrow 0} 0.5 \]
Now it suffices to bound the second term by an absolute constant smaller than $0.5$. Indeed 
\begin{align*}
    \frac{ \sum_{i=0}^{k-1} (1-\alpha)^{3i}}{( \sum_{i=0}^{k-1} (1-\alpha)^{2i})^{3/2}} &= \frac{ \frac{1- (1-\alpha)^{3k}}{1- (1-\alpha)^3} }{ \left( \frac{1- (1-\alpha)^{2k}}{1- (1-\alpha)^2} \right )^{3/2}} = \frac{h((1-\alpha)^k)}{h(1-\alpha)}
\end{align*}
where $h(x) = \frac{1-x^3}{(1-x^2)^{3/2}}$. Taking the derivative yields, 
\[ h'(x) = \frac{-3x^2(1-x^2) +3x(1-x^3)}{(1-x^2)^{5/2}} = \frac{3x(1-x)}{(1-x^2)^{5/2}} \ge 0\]
Therefore $h$ is increasing on [0,1]. Hence $h(1-x) \ge h(0) = 1$ and $h((1-\alpha)^k) \le h(\sqrt{\epsilon/\Delta})$. This leads to \[ \frac{ \sum_{i=0}^{k-1} (1-\alpha)^{3i}}{( \sum_{i=0}^{k-1} (1-\alpha)^{2i})^{3/2}}  \le h(\sqrt{\epsilon/\Delta}) \le 1 + 3\frac{\epsilon}{\Delta}\]
when $\epsilon$ is small enough. Hence when $m \ge 2$, 
\[ C \frac{ m \sum_{i=0}^{k-1} (1-\alpha)^{3i}}{( m \sum_{i=0}^{k-1} (1-\alpha)^{2i})^{3/2}} \le 0.3627(1 + 3\frac{\epsilon}{\Delta}) \rightarrow 0.3627 \]
Therefore, what we show is 
\[ m \le \frac{\alpha}{2\sqrt{2\epsilon}} \implies \liminf_{\epsilon \rightarrow 0} \PP \left (x_k -(1-\alpha)^{k}x_0 \ge \sqrt{2\epsilon} \right ) \ge 0.1\]
Hence we must have $m > \frac{\alpha}{2\sqrt{2\epsilon}}$. Together with the requirement $ k \ge  \frac{\log \left( \frac{\Delta}{\epsilon} \right )}{-2\log (1-\alpha)}$, we could bound the total iteration $T=km$ as in the expectation case.\\

{\bf{SVRG:}}\\
So far, we have proved the result for applying gradient descent with full gradient estimator. We are going to show that applying SVRG on the proposed function is yields indeed a full gradient, but with a different stepsize. This allows us to conclude since our previous result holds for any stepsize $\alpha$.

We fix a reference point $x_k$ and consider the inner loop with respect to $\tilde{x} = x_k$. Recall that the update of SVRG in the inner loop is given by
\begin{align*} 
x_t & = x_{t-1} - \eta ( \nabla f_{i_t}(x_{t-1}) - \nabla f_{i_t}(\tilde{x} ) + \tilde{\mu}). \\
    & = (1-\eta) x_{t-1} + \eta \tilde{x}  - \eta \tilde{\mu}
\end{align*}
Therefore, a simple recurrence leads to
\[ x_t = \tilde{x} - \eta (\sum_{i=0}^{t-1} (1-\eta)^i ) \tilde{\mu}\]
Hence, the next reference point is given by
\[ x_{k+1} = \frac{1}{\msvrg} \sum_{t=1}^{\msvrg} x_t = \tilde{x} - \tilde{\eta} \tilde{\mu}  = x_k - \tilde{\eta} g(x_k),\]
where $\tilde{\eta} =  \frac{\eta}{\msvrg} \sum_{t=0}^{\msvrg-1} (\msvrg -t)(1-\eta)^t $ does 
not depend on $k$ and $g(x_k)$ is the gradient estimator at $x_k$. This is exactly a GD with 
gradient estimator and stepsize $\tilde{\eta}$. Therefore, both the expectation and high 
probability lower bound follows from the result of GD. 
\end{proof}



